\newif\iftecrep\tecreptrue 

\iftecrep
\documentclass[12pt,twoside]{article}
\else
\documentclass[oribibl,envcountsame]{llncs}
\fi
\usepackage{latexsym,amsmath,amssymb,subfigure,graphicx}

\newcommand{\eqn}[1]{\begin{align}#1\end{align}}
\newcommand{\eq}[1]{\begin{align*}#1\end{align*}}

\newcommand{\argmin}{\operatornamewithlimits{arg\,min}}
\iftecrep
\newcommand{\proofof}{Proof of }
\newcommand{\qedifnone}{}
\else
\newcommand{\proofof}{}
\newcommand{\qedifnone}{\qed}
\fi

\def\paradot#1{\vspace{1ex plus 0.5ex minus 0.5ex}\noindent{\bf\boldmath{#1.}}}

\newcommand{\ind}[1]{[\![ #1 ]\!]}
\newcommand{\R}[0]{\mathbb R}
\newcommand{\N}[0]{\mathbb N}

\iftecrep
\usepackage{amsthm}
\theoremstyle{plain}
\newtheorem{theorem}{Theorem}
\newtheorem{proposition}{Proposition}
\newtheorem{lemma}[theorem]{Lemma}
\theoremstyle{definition}
\newtheorem{definition}[theorem]{Definition}

\newtheorem{example}[theorem]{Example}
\theoremstyle{remark}
\else
\spnewtheorem{assumption}{Assumption}{\bfseries}{\itshape}
\fi

\newenvironment{keywords}{\centerline{\bf\small
Keywords}\begin{quote}\small}{\par\end{quote}\vskip 1ex}

\newcommand{\X}[0]{X}

\newcommand{\Y}[0]{{Y}}
\newcommand{\M}[0]{\mathcal M}
\newcommand{\KM}[0]{K\!M}
\newcommand{\Ms}[0]{\mathbf M}
\newcommand{\Mn}[0]{\mathbf M_{norm}}
\newcommand{\s}[0]{;}
\newcommand{\B}[0]{\mathcal B}
\newcommand{\test}[0]{{\X_u}}
\newcommand{\train}[0]{{\X_m}}

\begin{document}
\iftecrep
\author{Tor Lattimore \and Marcus Hutter}
\title{
\vskip 2mm\bf\Large\hrule height5pt \vskip 4mm
No Free Lunch versus Occam's Razor in Supervised Learning
\vskip 4mm \hrule height2pt}
\author{{\bf Tor Lattimore}$^1$ and {\bf Marcus Hutter}$^{1,2,3}$\\[3mm]
\normalsize Research School of Computer Science \\[-0.5ex]
\normalsize $^1$Australian National University and $^2$ETH Z{\"u}rich and $^3$NICTA \\[-0.5ex]
\normalsize\texttt{\{tor.lattimore,marcus.hutter\}@anu.edu.au}
}
\date{15 November 2011}

\else
\title{No Free Lunch versus Occam's Razor in Supervised Learning}
\author{Tor Lattimore$^1$ \and Marcus Hutter$^{1,2,3}$}
\institute{
$^1$Australian National University and $^2$ETH Z{\"u}rich and $^3$NICTA\\
\email{\{tor.lattimore,marcus.hutter\}@anu.edu.au}
}
\fi
\maketitle

\begin{abstract}
The No Free Lunch theorems are often used to argue that domain
specific knowledge is required to design successful algorithms.
We use algorithmic information theory to argue the case for a
universal bias allowing an algorithm to succeed in all
interesting problem domains. Additionally,  we give a new
algorithm for off-line classification, inspired by Solomonoff
induction, with good performance on all structured
(compressible) problems under reasonable assumptions. This
includes a proof of the efficacy of the well-known heuristic of
randomly selecting training data in the hope of reducing the
misclassification rate.
\iftecrep\def\contentsname{\centering\normalsize Contents}
{\parskip=-2.7ex\tableofcontents}\fi
\end{abstract}

\begin{keywords}
Supervised learning;
Kolmogorov complexity;
no free lunch;
Occam's razor.
\end{keywords}

\iftecrep\newpage\fi

\section{Introduction}

The No Free Lunch (NFL) theorems, stated and proven in various
settings and domains \cite{Sch94,wol01,wol97}, show that no
algorithm performs better than any other when their performance
is averaged uniformly over all possible problems of a
particular type.\footnote{Such results have been less formally
discussed long before by Watanabe in 1969 \cite{Sat69}.} These
are often cited to argue that algorithms must be designed for a
particular domain or style of problem, and that there is no
such thing as a general purpose algorithm.

On the other hand, Solomonoff induction \cite{sol64a,sol64b}
and the more general AIXI model \cite{Hut04} appear to
universally solve the sequence prediction and reinforcement
learning problems respectively. The key to the apparent
contradiction is that Solomonoff induction and AIXI do not
assume that each problem is equally likely. Instead they apply
a bias towards more structured problems. This bias is universal
in the sense that no class of structured problems is favored
over another. This approach is philosophically well justified
by Occam's razor.

The two classic domains for NFL theorems are optimisation and
classification. In this paper we will examine classification
and only remark that the case for optimisation is more complex.
This difference is due to the active nature of optimisation
where actions affect future observations.

Previously, some authors have argued that the NFL theorems do
not disprove the existence of universal algorithms for two
reasons.
\begin{enumerate}
\item That taking a uniform average is not philosophically
    the right thing to do, as argued informally in
    \cite{Gir}.
\item Carroll and Seppi in \cite{Car07} note that the NFL
    theorem measures performance as misclassification rate,
    where as in practise, the utility of a
    misclassification in one direction may be more costly
    than another.
\end{enumerate}
We restrict our consideration to the task of minimising the
misclassification rate while arguing more formally for a
non-uniform prior inspired by Occam's razor and formalised by
Kolmogorov complexity. We also show that there exist algorithms
(unfortunately only computable in the limit) with very good
properties on all structured classification problems.

The paper is structured as follows. First, the required
notation is introduced (Section 2). We then state the original
NFL theorem, give a brief introduction to Kolmogorov
complexity, and show that if a non-uniform prior inspired by
Occam's razor is used, then there exists a free lunch (Section
3). Finally, we give a new algorithm inspired by Solomonoff
induction with very attractive properties in the classification
problem (Section 4).

\section{Preliminaries}

Here we introduce the required notation and define the problem
setup for the No Free Lunch theorems.

\paradot{Strings}
A finite string $x$ over alphabet $\X$ is a finite sequence
$x_1x_2x_3\cdots x_{n-1}x_n$ with $x_i \in \X$. An infinite
string $x$ over alphabet $\X$ is an infinite sequence
$x_1x_2x_3\cdots$.
Alphabets are usually countable or finite, while in this paper
they will almost always be binary.
For finite strings we have a length function defined by
$\ell(x) := n$ for $x = x_1x_2\cdots x_n$. The empty string of
length $0$ is denoted by $\epsilon$.
The set $\X^n$ is the set of all strings of length $n$. The set
$\X^*$ is the set of all finite strings. The set $\X^\infty$ is
the set of all infinite strings.
Let $x$ be a string (finite or infinite) then substrings are
denoted $x_{s:t} := x_s x_{s+1}\cdots x_{t-1} x_{t}$ where $s
\leq t$. A useful shorthand is $x_{<t} := x_{1:t-1}$.
Let $x,y \in \X^*$ and $z \in \X^\infty$ with $\ell(x) = n$ and
$\ell(y) = m$ then \eq{
xy &:= x_1x_2,\cdots x_{n-1}x_{n}y_1y_2\cdots y_{m-1}y_m \\
xz &:= x_1x_2,\cdots x_{n-1}x_{n}z_1z_2z_3\cdots
}
As expected, $xy$ is finite and has length $\ell(xy) = n + m$
while $xz$ is infinite.
For binary strings, we write $\#1(x)$ and $\#0(x)$ to mean the
number of 0's and number of 1's in $x$ respectively.

\paradot{Classification}
Informally, a classification problem is the task of matching
features to class labels. For example, recognizing handwriting
where the features are images and the class labels are letters.
In supervised learning, it is (usually) unreasonable to expect
this to be possible without any examples of correct
classifications. This can be solved by providing a list of
feature/class label pairs representing the true classification
of each feature. It is hoped that these examples can be used to
generalize and correctly classify other features.

The following definitions formalize classification problems,
algorithms capable of solving them, as well as the loss
incurred by an algorithm when applied to a problem, or set of
problems. The setting is that of transductive learning as in
\cite{DEM04}.
\begin{definition}[Classification Problem]
Let $\X$ and $\Y$ be finite sets representing the feature space
and class labels respectively. A {\it classification problem}
over $\X,\Y$ is defined by a function $f:\X\to\Y$ where $f(x)$
is the true class label of feature $x$.
\end{definition}
In the handwriting example, $\X$ might be the set of all images
of a particular size and $\Y$ would be the set of
letters/numbers as well as a special symbol for images that
correspond to no letter/number.
\begin{definition}[Classification Algorithm]
Let $f$ be a classification problem and $\train \subseteq \X$
be the training features on which $f$ will be known. We write
$f_\train$ to represent the function $f_\train:\train \to \Y$
with $f_\train(x) := f(x)$ for all $x\in\train$. A {\it
classification algorithm} is a function, $A$, where
$A(f_\train, x)$ is its guess for the class label of feature
$x\in\test := \X - \train$ when given training data $f_\train$.
Note we implicitly assume that $\X$ and $\Y$ are known to the
algorithm.
\end{definition}

\begin{definition}[Loss function]
The loss of algorithm $A$, when applied to classification
problem $f$, with training data $\train$ is measured by
counting the proportion of misclassifications in the testing
data, $\test$.
\eq{
L_A(f, \train):= {1\over|\test|} \sum_{x \in \test} \ind{A(f_\train, x) \neq f(x)}
}
where $\ind{}$ is the indicator function defined by,
$\ind{expr} = 1$ if $expr$ is true and $0$ otherwise.
\end{definition}

We are interested in the expected loss of an algorithm on the
set of all problems where expectation is taken with respect to
some distribution $P$.
\begin{definition}[Expected loss]
Let $\M$ be the set of all functions from $\X$ to $\Y$ and $P$
be a probability distribution on $\M$. If $\train$ is the
training data then the expected loss of algorithm $A$ is
\eq{
L_A(P, \train) := \sum_{f \in \M} P(f) L_A(f, \train)
}
\end{definition}

\section{No Free Lunch Theorem}

We now use the above notation to give a version of the No Free
Lunch Theorem of which Wolpert's is a generalization.
\begin{theorem}[No Free Lunch]\label{thm-nfl1}
Let $P$ be the uniform distribution on $\M$. Then the following
holds for any algorithm $A$ and training data
$\train\subseteq\X$.
\eqn{ \label{eqn-loss}
L_A(P, \train) = |\Y - 1| / |\Y|
}
\end{theorem}
The key to the proof is the following observation. Let $x \in
\test$, then for all $y \in \Y$, $P(f(x) = y | f|_\train) =
P(f(x) = y) = 1/|\Y|$. This means no information can be
inferred from the training data, which suggests no algorithm
can be better than random.

\paradot{Occam's razor/Kolmogorov complexity}
The theorem above is often used to argue that no general
purpose algorithm exists and that focus should be placed on
learning in specific domains.

The problem with the result is the underlying assumption that
$P$ is uniform, which implies that training data provides no
evidence about the true class labels of the test data. For
example, if we have classified the sky as blue for the last
1,000 years then a uniform assumption on the possible sky
colours over time would indicate that it is just as likely to
be green tomorrow as blue, a result that goes against all our
intuition.

How then, do we choose a more reasonable prior? Fortunately,
this question has already been answered heuristically by
experimental scientists who must endlessly choose between one
of a number of competing hypotheses. Given any experiment, it
is easy to construct a hypothesis that fits the data by using a
lookup table. However such hypotheses tend to have poor
predictive power compared to a simple alternative that also
matches the data. This is known as the principle of parsimony,
or Occam's razor, and suggests that simple hypotheses should be
given a greater weight than more complex ones.

Until recently, Occam's razor was only an informal heuristic.
This changed when Solomonoff, Kolmogorov and Chaitin
independently developed the field of algorithmic information
theory that allows for a formal definition of Occam's razor. We
give a brief overview here, while a more detailed introduction
can be found in \cite{LV08}. An in depth study of the
philosophy behind Occam's razor and its formalisation by
Kolmogorov complexity can be found in \cite{KLV97,HR11}. While
we believe Kolmogorov complexity is the most foundational
formalisation of Occam's razor, there have been other
approaches such as MML \cite{BW68} and MDL \cite{Gru07}. These
other techniques have the advantage of being computable (given
a computable prior) and so lend themselves to good practical
applications.

The idea of Kolmogorov complexity is to assign to each binary
string an integer valued {\it complexity} that represents the
length of its shortest description. Those strings with short
descriptions are considered simple, while strings with long
descriptions are complex. For example, the string consisting of
1,000,000 1's can easily be described as ``one million ones''.
On the other hand, to describe a string generated by tossing a
coin 1,000,000 times would likely require a description about
1,000,000 bits long. The key to formalising this intuition is
to choose a universal Turing machine as the language of
descriptions.

\begin{definition}[Kolmogorov Complexity]
Let $U$ be a universal Turing machine and $x \in \B^*$ be a
finite binary string. Then define the plain Kolmogorov
complexity $C(x)$ to be the length of the shortest program
(description) $p$ such that $U(p) = x$.
\eq{
C(x) := \min_{p\in\B^*} \left\{\ell(p) : U(p) = x\right\}
}
\end{definition}
It is easy to show that $C$ depends on choice of universal
Turing machine $U$ only up to a constant independent of $x$ and
so it is standard to choose an arbitrary {\it reference}
universal Turing machine.

For technical reasons it is difficult to use $C$ as a prior, so
Solomonoff introduced monotone machines to construct the
Solomonoff prior, $\Ms$. A monotone Turing machine has one
read-only input tape which may only be read from left to right
and one write-only output tape that may only be written to from
left to right. It has any number of working tapes. Let $T$ be
such a machine and write $T(p) = x$ to mean that after reading
$p$, $x$ is on the output tape. The machines are called
monotone because if $p$ is a prefix of $q$ then $T(p)$ is a
prefix of $T(q)$.
It is possible to show there exists a universal monotone Turing
machine $U$ and this is used to define monotone complexity $Km$
and Solomonoff's prior, $\Ms$.
\begin{definition}[Monotone Complexity]
Let $U$ be the reference universal monotone Turing machine then
define $Km$, $\Ms$ and $\KM$ as follows,
\eq{
Km(x) &:= \min\left\{\ell(p) : U(p) = x*\right\} \\
\Ms(x) &:= \sum_{U(p) = x*} 2^{-\ell(p)} \\
\KM(x)&:= -\log \Ms(x)
}
where $U(p) = x*$ means that when given input $p$, $U$ outputs
$x$ possibly followed by more bits.
\end{definition}
Some facts/notes follow.
\begin{enumerate}
\item For any $n$, $\sum_{x \in \B^n} \Ms(x) \leq 1$.
\item $Km$, $\Ms$ and $\KM$ are incomputable.
\item $0 < \KM(x) \approx Km(x) \approx C(x) < \ell(x) + O(1)$\footnote{
The approximation $C(x) \approx Km(x)$ is only accurate to $\log \ell(x)$, while $\KM \approx Km$ is almost always very close \cite{gacs83,Gacs08}. This is a little surprising since the sum in the definition of $\Ms$ contains $2^{-Km}$. It shows that there are only comparitively few
short programs for any $x$.
}
\end{enumerate}
To illustrate why $\Ms$ gives greater weight to simple $x$,
suppose $x$ is simple then there exists a relatively short
monotone Turing machine $p$, computing it. Therefore $Km(x)$ is
small and so $2^{-Km(x)} \approx \Ms(x)$ is relatively large.

Since $\Ms$ is a semi-measure rather than a proper measure, it
is not appropriate to use it in place of $P$ when computing
expected loss. However it can be normalized to a proper
measure, $\Mn$ defined inductively by
\eq{
\Mn(\epsilon) &:= 1 & \Mn(x b) &:= \Mn(x) {\Ms(x b) \over \Ms(x 0) + \Ms(x 1)}
}
Note that this normalisation is not unique, but is
philosophically and technically the most attractive and was
used and defended by Solomonoff. For a discussion of
normalisation, see \cite[p.303]{LV08}. The normalised version
satisfies $\sum_{x \in \B^n} \Mn(x) = 1$.

We will also need to define $\Ms/\KM$ with side information,
$\Ms(y \s x) := \Ms(y)$ where $x*$ is provided on a spare tape
of the universal Turing machine. Now define $\KM(y \s x) :=
-\log \Ms(y \s x)$.
This allows us to define the complexity of a function in terms
of its output relative to its input.
\begin{definition}[Complexity of a function]
Let $\X=\left\{x_1,\cdots,x_n\right\} \subseteq \B^k$ and
$f:\X\to\B$ then define the complexity of $f$, $\KM(f \s \X)$
by
\eq{
\KM(f \s \X) := \KM(f(x_1)f(x_2)\cdots f(x_n) \s x_1,x_2,\cdots ,x_n)
}
\end{definition}
An example is useful to illustrate why this is a good measure
of the complexity of $f$.
\begin{example}
Let $\X \subseteq \B^n$ for some $n$, and $\Y = \B$ and
$f:\X\to\Y$ be defined by $f(x) = \ind{x_n = 1}$. Now for a
complex $\X$, the string $f(x_1)f(x_2)\cdots $ might be
difficult to describe, but there is a very short program that
can output $f(x_1)f(x_2)\cdots$ when given $x_1x_2\cdots$ as
input. This gives the expected result that $\KM(f;\X)$ is very
small.
\end{example}

\paradot{Free lunch using Solomonoff prior}
We are now ready to use $\Mn$ as a prior on a problem family.
The following proposition shows that when problems are chosen
according to the Solomonoff prior that there is a (possibly
small) free lunch.

Before the proposition, we remark on problems with maximal
complexity, $\KM(f;\X) = O(|\X|)$. In this case $f$ exhibits no
structure allowing it to be compressed, which turns out to be
equivalent to being random in every intuitive sense
\cite{Lof66}. We do not believe such problems are any more
interesting than trying to predict random coin flips. Further,
the NFL theorems can be used to show that no algorithm can
learn the class of random problems by noting that almost all
problems are random. Thus a bias towards random problems is not
much of a bias (from uniform) at all, and so at most leads to a
decreasingly small free lunch as the number of problems
increases.

\begin{proposition}[Free lunch under Solomonoff prior]\label{prop_main}
Let $\Y = \B$ and fix a $k \in \N$. Now let $\X = \B^n$ and
$\train \subset \X$ such that $|\train| = 2^{n} - k$. For
sufficiently large $n$ there exists an algorithm $A$ such that
\eq{
L_A(\Mn, \train) < 1/2
}
\end{proposition}

\iftecrep
Before the proof of Proposition \ref{prop_main}, we require an easy lemma.
\begin{lemma}
Let $N \subset \M$ then there exists an algorithm $A_N$ such that
\eq{
\sum_{f \in N} P(f) L_{A_N}(f, \train) \leq {1 \over 2} \sum_{f \in N} P(f)
}
\end{lemma}

\begin{proof}
Let $A_i$ with $i \in \left\{0, 1\right\}$ be the algorithm
always choosing $i$. Note that
\eq{
\sum_{f \in N} P(f) L_{A_0}(f, \train) = \sum_{f \in N} P(f) (1 - L_{A_1}(f, \train))
}
The result follows easily.
\qedifnone\end{proof}

\begin{proof}[\proofof Proposition \ref{prop_main}]
Now let $\M_1$ be the set of all $f \in \M$ with $f(y) = 1
\forall y \in \train$ and $\M_0 = \M - \M_1$. Now construct an
$A$ by
\eq{
A(f_\train, x) = \begin{cases}
1 & \text{if } f \in \M_1 \\
A_{\M_0}(f_\train, x) & \text {otherwise}
\end{cases}
}
Let $f_1 \in \M_1$ be the constant valued function such that
$f_1(x) = 1 \forall x$ then
\eqn{
\label{fl-1} L_A(\Mn, \train) &= \sum_{f \in \M} \Mn(f) L_A(f, \train) \\
\label{fl-2} &= \sum_{f \in \M_0} \Mn(f) L_A(f, \train) + \sum_{f \in \M_1} \Mn(f) L_A(f, \train) \\
\label{fl-3}&\leq {1 \over 2} \sum_{f \in \M_0} \Mn(f) + \sum_{f \in \M_1} \Mn(f) L_A(f, \train) \\
\label{fl-4}&\leq {1 \over 2} \sum_{f \in \M_0} \Mn(f) + \sum_{f \in \M_1 - f_1} \Mn(f) \\
\label{fl-5}&< {1 \over 2} (1 - \delta) + \sum_{f \in \M_1 - f_1} \Mn(f)
< {1 \over 2}
}
where (\ref{fl-1}) is definitional, (\ref{fl-2}) follows by
splitting the sum into $\M_0$ and $\M_1$, (\ref{fl-3}) by the
previous lemma, (\ref{fl-4}) since loss is bounded by $1$ and
the loss incurred on $f_1$ is $0$. The first inequality of
(\ref{fl-5}) follows since it can be shown that there exists a
$\delta > 0$ such that $\Mn(f_1) > \delta$ with $\delta$
independent of $n$. The second because $\max_{f \in \M_1 -
\left\{f_1\right\}} \Mn(f) \stackrel{n \to \infty}
\longrightarrow 0$ and $|\M_1|$ is independent of $n$.
\qedifnone\end{proof}
\else
The proof is omitted due to space limitations, but the idea is
very simple. Consider the algorithm such that $A(f|_\train, x)
= 1$ if $f(x) = 1$ for all $x \in \train$ and $A(f|_\train, x)$
is random otherwise. Then show that if the amount of training
data is extremely large relative to the testing data then the
Solomonoff prior assigns greater weight to the function $f_1(x)
:= 1$ for all $x$ than the set of functions satisfying $f(x) =
1$ for all $x \in \train$ but $f(x) \neq 1$ for some $x \in
\test$.
\fi

The proposition is unfortunately extremely weak. It is more
interesting to know exactly what conditions are required to do
much better than random. In the next section we present an
algorithm with good performance on all well structured problems
when given ``good'' training data. Without good training data,
even assuming a Solomonoff prior, we believe it is unlikely
that the best algorithm will perform well.

Note that while it appears intuitively likely that any
non-uniform distribution such as $\Mn$ might offer a free
lunch, this is in fact not true. It is shown in \cite{Sch01}
that there exist non-uniform distributions where the loss over
a problem family is independent of algorithm. These
distributions satisfy certain symmetry conditions not satisfied
by $\Mn$, which allows Proposition \ref{prop_main} to hold.

\section{Complexity-based classification}

Solomonoff induction is well known to solve the online
prediction problem where the true value of each classification
is known after each guess. In our setup, the true
classification is only known for the training data, after which
the algorithm no longer receives feedback. While Solomonoff
induction can be used to bound the number of total errors while
predicting deterministic sequences, it gives no indication of
when these errors may occur. For this reason we present a
complexity-inspired algorithm with better properties for the
offline classification problem.

Before the algorithm we present a little more notation. As
usual, let $\X = \left\{x_1,x_2,\cdots, x_n\right\} \subseteq
\B^k$, $\Y = \B$ and let $\train \subseteq \X$ be the training
data. Now define an indicator function $\chi$ by $\chi_i :=
\ind{x_i \in \train}$.
\begin{definition}
Let $f \in \Y^\X$ be a classification problem. The algorithm
$A^*$ is defined in two steps.
\eq{
\tilde f &:= \argmin_{\tilde f \in \Y^\X} \left\{\KM(\tilde f ; \X) : \chi_i = 1 \implies \tilde f(x_i) = f(x_i) \right\}  \\
A^*(f_\train, x_i) &:= \tilde f(x_i)
}
\end{definition}
Essentially $A^*$ chooses for its model the simplest $\tilde f$
consistent with the training data and uses this for classifying
unseen data. Note that the definition above only uses the value
$y_i = f(x_i)$ where $\chi_i = 1$, and so it does not depend on
unseen labels.

If $\KM(f \s \X)$ is ``small'' then the function we wish to
learn is simple so we should expect to be able to perform good
classification, even given a relatively small amount of
training data. This turns out to be true, but only with a good
choice of training data. It is well known that training data
should be ``broad enough'', and this is backed up by the
example below and by Theorem \ref{thm-class-alg}, which give an
excellent justification for random training data based on good
theoretical (Theorem \ref{thm-class-alg}) and philosophical
(AIT) underpinnings. The following example demonstrates the
effect of bad training data on the performance of $A^*$.

\setlength\abovecaptionskip{0cm}
\setlength\belowcaptionskip{0cm}
\setlength\floatsep{0cm}
\setlength\textfloatsep{0cm}
\setlength\intextsep{0.3cm}
\begin{figure}[h!]
\begin{center}
\includegraphics{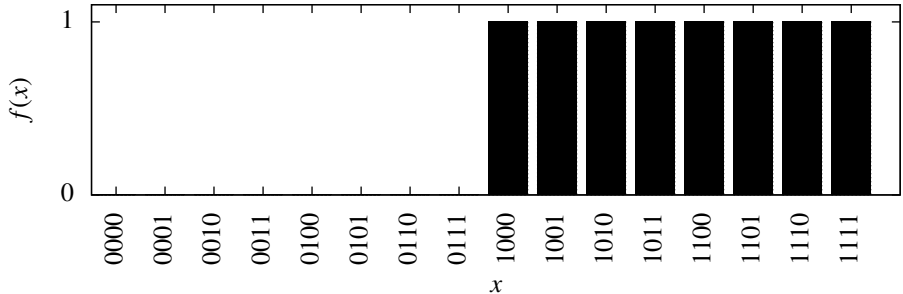}
\end{center}
\caption{A simple problem}
\label{fig-simple}
\end{figure}
\begin{example}
Let $\X = \left\{0000,0001,0010, 0011, \cdots, 1101, 1110,
1111\right\}$ and $f(x)$ be defined to be the first bit of $x$
as in Figure \ref{fig-simple}. Now suppose $\chi = 1^{8} 0^{8}$
(So the algorithm is only allowed to see the true class labels
of $x_1$ through $x_{8}$). In this case, the simplest $\tilde
f$ consistent with the first 16 data points, all of which are
zeros, is likely to be $\tilde f(x) = 0$ for all $x \in \X$ and
so $A^*$ will fail on every piece of testing data!

On the other hand, if $\chi = 001010011101101$, which was
generated by tossing a coin 16 times, then $\tilde f$ will very
likely be equal to $f$ and so $A^*$ will make no errors. Even
if $\chi$ is zero about the critical point in the middle
($\chi_{8} = \chi_{9} = 0$) then $\tilde f$ should still match
$f$ mostly around the left and right and will only be unsure
near the middle.

Note, the above is not precisely true since for small strings
the dependence of $\KM$ on the universal monotone Turing
machine can be fairly large. However if we increase the size of
the example so that $|\X| > 1000$ then these quirks disappear
for natural reference universal Turing machines.
\end{example}
\begin{definition}[Entropy]
Let $\theta\in[0,1]$
\eq{
H(\theta):=
\begin{cases}
-[ \theta \log \theta + (1-\theta)\log (1-\theta)] & \text{if } \theta \neq 0 \text{ and } \theta\neq 1 \\
0 & \text{otherwise}
\end{cases}
}
\end{definition}
\begin{theorem}\label{thm-class-alg}
Let $\theta \in (0, 1)$ be the proportion of data to be given for training then:
\begin{enumerate}
\item There exists a $\chi \in \B^\infty$ (training set) such that for all $n \in \N$, $\theta n - c_1 < \#1(\chi_{1:n}) < \theta n + c_1$ and
$n H(\theta) - c_2 < \KM(\chi_{1:n})$ for some $c_1,c_2 \in \R^+$.
\item For $n = |\X|$, the loss of algorithm $A^*$ when using training data determined by $\chi$ is bounded by
\eq{
L_{A^*}(f, \train) < {2\KM(f;\X) + \KM(\X) + c_2 + c_3 \over {n(1-\theta-c_1/n)\log (1-\theta+c_1/n)^{-1}}}
}
where $c_3$ is some constant independent of all inputs.
\end{enumerate}
\end{theorem}
This theorem shows that $A^*$ will do well on all problems
satisfying $\KM(f;\X) = o(n)$ when given good (but not
necessarily a lot) of training data. Before the proof, some
remarks.
\begin{enumerate}
\item The bound is a little messy, but for small $\theta$,
    large $n$ and simple $\X$ we get $L_{A^*}(f, \train)
    \stackrel{\approx}< {2\KM(f;\X) / (n\theta)}$.
\item The loss bound is extremely bad for large $\theta$.
    We consider this unimportant since we only really care
    if $\theta$ is small. Also, note that if $\theta$ is
    large then the number of points we have to classify is
    small and so we still make only a few mistakes.
\item The constants $c_1, c_2$ and $c_3$ are relatively
    small (around 100-500). They represent the length of
    the shortest programs computing simple transformations
    or encodings. This {\it is} dependent on the universal
    Turing machine used to define the Solomonoff
    distribution, but for a {\it natural} universal Turing
    machine we expect it to be fairly small
    \cite[sec.2.2.2]{Hut04}.
\item The ``special'' $\chi$ is not actually that special
    at all. In fact, it can be generated easily with
    probability 1 by tossing a coin with bias $\theta$
    infinitely often. More formally, it is a $\mu$
    Martin-L\"of random string where $\mu(1 | x) = \theta$
    for all $x$. Such strings form a $\mu$-measure 1 set in
    $\B^\infty$.
\end{enumerate}
\begin{proof}[\proofof Theorem \ref{thm-class-alg}]
The first is a basic result in algorithmic information theory
\cite[p.318]{LV08}. Essentially choosing $\chi$ to be
Martin-L\"of random with respect to a Bernoulli process
parameterized by $\theta$. From now on, let $\bar \theta =
\#1(\chi) / n$. For simplicity we write $x := x_1x_2\cdots
x_n$, $y := f(x_1)f(x_2)\cdots f(x_n)$, and $\tilde y := \tilde
f(x_1) \tilde f(x_2)\cdots \tilde f(x_n)$.
Define indicator $\psi$ by $\psi_i := \ind{\chi_i = 0 \wedge
y_i = \tilde y_i}$. Now note that there exists $c_3 \in \R$
such that
\eqn{
\label{eqn_coding}\KM(\chi_{1:n}) < \KM(\psi_{1:n};y, \tilde y) + \KM(y;x) + \KM(\tilde y;x) + \KM(x) + c_3
}
This follows since we can easily use $y$, $\tilde y$ and
$\psi_{1:n}$ to recover $\chi_{1:n}$ by $\chi_i = 1$ if and
only if $y_i = \tilde y_i$ and $\psi_i \neq 1$. The constant
$c_3$ is the length of the reconstruction program.
Now $\KM(\tilde y ; x) \leq \KM(y ; x)$ follows directly from
the definition of $\tilde f$. We now compute an upper bound on
$\KM(\psi)$.
Let $\alpha := L_{A^*}(f, \train)$ be the proportion of the
testing data on which $A^*$ makes an error. The following is
easy to verify:
\begin{enumerate}
\item $\#1(\psi) = (1 - \alpha)(1 - \bar\theta) n$
\item $\#0(\psi) = (1 - (1-\alpha)(1-\bar\theta)) n$
\item $y_i \neq \tilde y_i \implies \psi_i = 0$
\item $\#1(y \oplus \tilde y) = \alpha(1-\bar\theta) n$ where $\oplus$ is the exclusive or function.
\end{enumerate}
We can use point 3 above to trivially encode $\psi_i$ when
$\tilde y_i \neq y_i$. Aside from these, there are exactly
$\bar\theta n$ 0's and $(1 - \alpha)(1 - \bar\theta)n$ 1's.
Coding this subsequence using frequency estimation gives a code
for $\psi_{1:n}$ given $y$ and $\tilde y$, which we substitute
into (\ref{eqn_coding}).
\eqn{
\nonumber nH(\bar\theta) - c_2 &\leq \KM(\chi_{1:n}) \leq
\KM(\psi_{1:n} ; y, \tilde y) + \KM(y;x) + \KM(\tilde y;x) \\
&\qquad + \KM(x) + c_3 \\
\nonumber &\leq 2 \KM(y;x) + \KM(x) + n J(\bar\theta, \alpha) + c_3
}
where $J(\bar \theta, \alpha) := \left[\bar\theta +
(1-\bar\theta)(1-\alpha)\right] H\left({\bar\theta /\left[
\bar\theta + (1-\bar\theta)(1-\alpha)\right]}\right)$. An easy
technical result (Lemma \ref{lem1} in the appendix) shows that
for $\bar\theta \in (0, 1)$
\eq{
0 \leq \alpha(1 - \bar\theta)\log {1 \over 1-\bar\theta} \leq
H(\bar\theta) - J(\bar\theta, \alpha)
}
Therefore $n \alpha(1-\bar\theta)\log {1 \over 1-\bar\theta}
\leq 2 \KM(y;x) + \KM(x) + c_2 + c_3$. The result follows by
rearranging and using part 1 of the theorem.
\qedifnone\end{proof}
Since the features are known, it is unexpected for the bound to
depend on their complexity, $\KM(\X)$. Therefore it is not
surprising that this dependence can be removed at a small cost,
and with a little extra effort.
\begin{theorem}\label{thm-class-alg2}
Under the same conditions as Theorem \ref{thm-class-alg}, the
loss of $A^*$ is bounded by
\eq{
L_{A^*}(f, \train) < {2\KM(f;\X) + 2 \left[\log |\X| + \log \log |\X| \right] + c \over {n(1-\theta-c_1/n)\log (1-\theta+c_1/n)^{-1}}}
}
where $c$ is some constant independent of inputs.
\end{theorem}
This version will be preferred to Theorem \ref{thm-class-alg}
in cases where $\KM(\X) > 2 \left[\log |\X| + \log\log
|\X|\right]$. The
proof of Theorem \ref{thm-class-alg2} is almost identical to that of Theorem \ref{thm-class-alg}.\\
{\it Proof sketch:}
The idea is to replace equation (\ref{eqn_coding}) by
\eqn{
\label{eqn_coding2} \KM(\chi_{1:n}, x) < \KM(\psi_{1:n};y, \tilde y) + \KM(y;x) + \KM(\tilde y;x) + \KM(x) + c_3
}
Then use the following identities $K(\chi_{1:n} ; x, K(x)) +
K(x) < K(\chi_{1:n}, x) - K(\ell(x)) < \KM(\chi_{1:n}, x)$
where the inequalities are true up to constants independent of
$x$ and $\chi$. Next a counting argument in combination with
Stirling's approximation can be used to show that for most
$\chi$ satisfying the conditions in Theorem \ref{thm-class-alg}
have $\KM(\chi_{1:n}) < K(\chi_{1:n}) < K(\chi_{1:n} ; x, K(x))
+ \log\ell(x) + r$ for some constant $r > 0$ independent of $x$
and $\chi$.
Finally use $\KM(x) < K(x)$ for all $x$ and $K(\ell(x)) < \log
\ell(x) + 2 \log \log \ell(x) + r$ for some constant $r > 0$
independent of $x$ to rearrange (\ref{eqn_coding2}) into
\eq{
\KM(\chi_{1:n}) &< \KM(\psi_{1:n} ; y, \tilde y) + \KM(y ; x) + \KM(\tilde y;x) + 2 \log \ell(x)   \\
&\qquad + 2 \log \log \ell(x) + c
}
for some constant $c > 0$ independent of $\chi, \psi, x$ and
$y$. Finally use the techniques in the proof of Theorem
\ref{thm-class-alg} to complete the proof. \qed

\section{Discussion}

\paradot{Summary}
Proposition \ref{prop_main} shows that if problems are
distributed according to their complexity, as Occam's razor
suggests they should, then a (possibly small) free lunch
exists. While the assumption of simplicity still represents a
bias towards certain problems, it is a universal one in the
sense that no style of structured problem is more favoured than
another.

In Section 4 we gave a complexity-based classification
algorithm and proved the following properties:
\begin{enumerate}
\item It performs well on problems that exhibit some compressible structure, $\KM(f;\X) = o(n)$.
\item Increasing the amount of training data decreases the error.
\item It performs better when given a good (broad/randomized) selection of training data.
\end{enumerate}
Theorem \ref{thm-class-alg} is reminiscent of the transductive
learning bounds of Vapnik and others \cite{DEM04,Vap82,Vap00},
but holds for {\it all} Martin-L\"of random training data,
rather than with high probability. This is different to the
predictive result in Solomonoff induction where results hold
with probability 1 rather than for all Martin-L\"of random
sequences \cite{HM07}. If we assume the training set is sampled
randomly, then our bounds are comparable to those in
\cite{DEM04}.

Unfortunately, the algorithm of Section 4 is incomputable.
However Kolmogorov complexity can be approximated via standard
compression algorithms, which may allow for a computable
approximation of the classifier of Section 4. Such
approximations have had some success in other areas of AI,
including general reinforcement learning \cite{HKUV11} and
unsupervised clustering \cite{CV05}.

Occam's razor is often thought of as the principle of choosing
the simplest hypothesis matching your data. Our definition of
simplest is the hypothesis that minimises $\KM(f;X)$ (maximises
$M(f;X)$). This is perhaps not entirely natural from the
informal statement of Occam's razor, since $M(x)$ contains
contributions from all programs computing $x$, not just the
shortest.
We justify this by combining Occam's razor with Epicurus
principle of multiple explanations that argues for all
consistent hypotheses to be considered. In some ways this is
the most natural interpretation as no scientist would entirely
rule out a hypothesis just because it is slightly more complex
than the simplest. A more general discussion of this issue can
be found in \cite[sec.4]{Dow11}. Additionally, we can argue
mathematically that since $\KM \approx Km$, the simplest
hypothesis is very close to the mixture.\footnote{The bounds of
Section 4 would depend on the choice of complexity at most
logarithmically in $|\X|$ with $\KM$ providing the uniformly
better bound.} Therefore the debate is more philosophical than
practical in this setting.

An alternative approach to formalising Occam's razor has been
considered in MML \cite{BW68}. However, in the deterministic
setting the probability of the data given the hypothesis
satisfies $P(D|H) = 1$. This means the two part code reduces to
the code-length of the prior, $\log (1/P(H))$. This means the
hypothesis with minimum message length depends only on the
choice of prior, not the complexity of coding the data. The
question then is how to choose the prior, on which MML gives no
general guidance.
Some discussion of Occam's razor from a Kolmogorov complexity
viewpoint can be found in \cite{Hut10,KLV97,HR11}, while the
relation between MML and Kolmogorov complexity is explored in
\cite{DW99}.

\paradot{Assumptions}
We assumed finite $\X$, $\Y$, and deterministic $f$, which is
the standard transductive learning setting. Generalisations to
countable spaces may still be possible using complexity
approaches, but non-computable real numbers prove more
difficult. One can either argue by the strong Church-Turing
thesis that non-computable reals do not exist, or approximate
them arbitrarily well. Stochastic $f$ are interesting and we
believe a complexity-based approach will still be effective,
although the theorems and proofs may turn out to be somewhat
different.

\paradot{Acknowledgements}
We thank Wen Shao and reviewers for valuable feedback on
earlier drafts and the Australian Research Council for support
under grant DP0988049.

\iftecrep
\appendix
\section{Technical proofs}

\begin{lemma}[\proofof Entropy inequality]\label{lem1}
\eqn{
\label{lem1-eqn1} 0 &\leq \alpha(1-\theta) \log {1 \over {1-\theta}} \\
\label{lem1-eqn2} &\leq H(\theta) - \left[\theta + (1-\theta)(1-\alpha)\right] H\left({\theta \over \theta + (1-\theta)(1-\alpha)}\right)
}
With equality only if $\theta \in \left\{0, 1\right\}$ or $\alpha = 0$
\end{lemma}
\begin{proof}
First, (\ref{lem1-eqn1}) is trivial. To prove (\ref{lem1-eqn2}), note that
for $\alpha = 0$ or $\theta \in \left\{0, 1\right\}$, equality is obvious. Now, fixing
$\theta \in (0, 1)$ and computing.
\eq{
& {\partial \over {\partial \alpha}} \left[H(\theta) - \left[\theta + (1-\theta)(1-\alpha)\right] H\left({\theta \over \theta + (1-\theta)(1-\alpha)}\right)\right] \\
&= (1-\theta) \log {{ 1 - \alpha(1-\theta)} \over {(1-\alpha)(1-\theta)}} \\
&\geq (1-\theta) \log (1-\theta)^{-1}
}
Therefore integrating both sides over $\alpha$ gives,
\eq{
\alpha (1-\theta) \log (1-\theta)^{-1} \leq H(\theta) - \left[\theta + (1-\theta)(1-\alpha)\right] H\left({\theta \over \theta + (1-\theta)(1-\alpha)}\right)
}
as required.
\qedifnone\end{proof}
\fi


\begin{small}
\newcommand{\etalchar}[1]{$^{#1}$}

\end{small}

\end{document}